\documentclass{article}

    \usepackage[preprint,nonatbib]{neurips_2020}

\usepackage[utf8]{inputenc} 
\usepackage[T1]{fontenc}    
\usepackage{url}            
\usepackage{booktabs}       
\usepackage{nicefrac}       
\usepackage{microtype}      

\usepackage{graphicx}
\usepackage{subfigure}
\usepackage{booktabs} 

\usepackage{amsmath, amssymb,amsthm, amsopn,stmaryrd, amsfonts, stmaryrd,mathtools}
\usepackage{bbold}
\usepackage{enumerate}
\usepackage{float}
\usepackage [left,pagewise]{lineno}
\usepackage {xspace}
\usepackage{ragged2e}
\usepackage{strivedi}

\usepackage{framed}
\newenvironment{aside}[1]{%
	\definecolor{shadecolor}{gray}{0.9}%
	\begin{shaded}{\color{Maroon}\noindent\textsc{#1}}\\%
	}{%
	\end{shaded}%
}

\title{The Expected Jacobian Outerproduct: Theory and Empirics}

\author{%
  S. Trivedi \\
  CS \& AI Lab \\
  MIT \\
  Cambridge, MA 02139 \\
  \And
  J. Wang \\
  Dept. of Computer Science \\
  The University of Chicago \\
  Chicago, IL 60637
}

\begin{document}

\maketitle

\begin{abstract}
The expected gradient outerproduct (EGOP) of an unknown regression function is an operator that arises in the theory of multi-index regression, and is known to recover those directions that are most relevant to predicting the output. However, work on the EGOP, including that on its cheap estimators, is restricted to the regression setting. In this work, we adapt this operator to the multi-class setting, which we dub the expected Jacobian outerproduct (EJOP). Moreover, we propose a simple rough estimator of the EJOP and show that somewhat surprisingly, it remains statistically consistent under mild assumptions. Furthermore, we show that the eigenvalues and eigenspaces also remain consistent. Finally, we show that the estimated EJOP can be used as a metric to yield improvements in real-world non-parametric classification tasks: both by its use as a metric, and also as cheap initialization in metric learning tasks.
\end{abstract}

\section{Introduction}
\label{sec:intro}

In high-dimensional classification and regression problems, the task is to infer the unknown function $f$ with $Y \approx f(X)$, given a set of observations $(\mathbf{x}, \mathbf{y})_i, i = 1, 2, \dots n.$, with $\mathbf{x}_i \in \mathcal{X} \subset \mathbb{R}^d$ and labels $\mathbf{y}_i$ being noisy versions of the function values $f(\mathbf{x}_i)$. We are interested in distance based (non-parametric) regression, which provides our function estimate $f_n (\mathbf{x}) = \sum_{i=1}^{n} w(\mathbf{x}, \mathbf{x}_i) \mathbf{y}_i$, where $w(\mathbf{x}, \mathbf{x}_i)$ depends on the distance $\rho$, defined as $\displaystyle \rho(\mathbf{x}, \mathbf{x}') = \sqrt{(\mathbf{x} - \mathbf{x}') \mathbf{W} (\mathbf{x} - \mathbf{x}') } $ and $\mathbf{W} \succeq 0$. The problem of estimating unknown $f$ becomes significantly harder as $d$ increases due to the \emph{curse of dimensionality}. To remedy this situation, several pre-processing techniques are used, each of which rely on a suitable assumption about the data and/or about $f$. For instance, a conceptually simple, yet often reasonable assumption that can be made is that $f$ might not vary equally along all coordinates of $\mathbf{x}$. Letting $f'_i = \nabla f^T e_i $ denote the derivative along coordinate $i$, and $\|f'_i\|_{1,\mu} \equiv \mathbb{E}_{\mathbf{x} \sim \mu} f'_i(\mathbf{x})$, we can use the above distance based estimator by setting $\rho$ such that $\mathbf{W}_{i,j} = \|f'_i\|_{1,\mu}$ when $i = j$ and $0$ otherwise. This \emph{gradient weighting} rescales the space such the ball $\mathcal{B}_{\rho}$ contains more points relative to the Euclidean ball $\mathcal{B}$. This is the intuition pursued in works such as \cite{DBLP:conf/nips/KpotufeB12}, \cite{GWJMLR}, with an emphasis on deriving an efficient, yet consistent estimator for the gradient. Using gradient weights for coordinate scaling in this manner has strong theoretical grounding: it has the effect of reducing the regression variance, while keeping the bias in control. The same assumption that $f$ may not equally vary along all coordinates is also the motivation for a plethora of variable selection methods. In the simplest setting, we use $f(X) = g(PX)$,  where $P\in \braces{0, 1}^{k \times d}$ projects $X$ down to $k<d$ coordinates that are most relevant to predicting the output. This idea is relaxed further in \emph{multi-index regression} e.g.\cite{li1991sliced,powell1989semiparametric, hardle1993optimal, xia2002adaptive}), by letting $P\in \real^{k \times d}$, which projects $X$ down to a $k$-dimensional subspace of $\real^d$. The motivation for multi-index regression is that while that while $f$ might vary along all coordinates, it actually may only depend on an unknown $k$-dimensional subspace, called a \emph{relevant} subspace. The task then becomes finding the said relevant subspace rather than chopping coordinates since they all might be relevant in predicting the output $y$. Work to recover this relevant subspace (which is sometimes also referred to in the literature as \emph{effective dimension reduction} \cite{li1991sliced}) gives rise to the expected gradient outerproduct (EGOP) \cite{trivedi2014UAI}: $\expectation_{X} G(X) \triangleq \expectation_X\paren{\nabla f(X) \cdot \nabla f(X)^\top}.$
Notice that if $f$ does not vary along some direction $\mathbf{v} \in \mathbb{R}^d$, then $\mathbf{v}$ must lie in the nullspace of the operator. Furthermore, it is easy to see that under mild assumptions on $f$, the column space of the EGOP is exactly the relevant subspace described above. Since the EGOP recovers the average variation of $f$ in all directions, it is also useful beyond the multi-index motivation. That is, even in the case when there isn't a clear relevant subspace, it is still reasonable to assume that $f$ may not vary equally in all directions. Thus, the EGOP can be used to weigh any direction $v\in \real^d$ according to its relevance as captured by the average variation of $f$ along $v$. Letting $VDV^\top$ be the spectral decomposition of the estimated EGOP, we can use it to transform the input $\mathbf{x}$ as $D^{1/2}V^\top \mathbf{x}$, which can be used as a distance function in non-parametric regression tasks. 

However, all prior work on multi-index regression and cognate topics only revolves around regression (and binary classification), in part due to the complexity of analysis. In this work we attack the more general multi-class case, which is treated as a multinomial regression problem with $c$ outputs, with the unknown function denoted as $f: \mathbb{R}^d \to \mathbb{S}^c$ where $\mathbb{S}^c = \{ \mathbf{y} \in \mathbb{R}^c | \mathbf{y} \geq 0, \mathbf{y}^T \mathbf{1} = 1 \}$. We are led to an operator similar to the EGOP based on computing the Jacobian of $f$, which we call the Expected Jacobian Outer Product (EJOP) $\mathbb{E}_{X} G(X) \triangleq  \expectation_\mathbf{x}\paren{\mathbf{J}_{f}(X) \cdot \mathbf{J}_{f}(X)^T}$. For constructing the EJOP, we need to compute gradient estimates, for which optimal estimators can be expensive in practice. We propose a simple, efficient, difference based estimator and show that despite it's crudity it remains statistically consistent under mild assumptions. This approach is also online and cheap: we only require $2d$ estimates of the function $f$ at $\mathbf{x}$. We also show that the EJOP can be used for metric weighing for distance-based non-parametric classification, as well as used as a pre-processing metric for standard metric learning tasks. 

\section{The Expected Jacobian Outerproduct}
Recall that in high dimensional classification problems over $\mathbb{R}^d$, the unknown (multinomial regression) function $f$ is a vector-valued function mapping to a probability simplex $\mathbb{S}^c = \{ \mathbf{y} \in \mathbb{R}^c | \textbf{  }\forall i \textbf{  } y_i \geq 0, \mathbf{y}^T \mathbf{1} = 1 \}$, where $c$ is the number of outputs. For classification, the prediction for some point $\mathbf{x}$ is then given by: $y = \argmax_{i = 1, \dots, c} f_i(\mathbf{x})$. For $f$, at point $\mathbf{x}$, denoting the Jacobian as $\mathbf{J}_{f}$, we are interested in the quantity $\mathbf{J}_{f}(\mathbf{x}) \mathbf{J}_{f}(\mathbf{x})^T$. Let $f_n$ be an initial estimate of $f$, for which we use a kernel estimate, then for the $(i,j)^{th}$ element of $\mathbf{J}_f(\mathbf{x})$, we can use the following rough estimate:
$$\displaystyle \Delta_{t,i,j} f_n(\mathbf{x}) = \frac{f_{n,i} (\mathbf{x} + t \mathbf{e}_j) - f_{n,i} (\mathbf{x} - t \mathbf{e}_j) }{2t}, t > 0$$
Let $G_n (\mathbf{x})$ be the outerproduct using the estimated Jacobian, the EJOP is estimated as $\mathbb{E}_n G_n (X)$. 

\subsection{Function Estimate}

First, we need to specify the function estimate, that is used both for the theoretical analysis, and the experiments reported. We denote the vector-valued function estimate by $\bar{f}_{n,h} (\mathbf{x}) \in \mathbb{S}^c$, for which we employ the estimate using an admissible kernel $K$: $\bar{f}_{n,h,c} (\mathbf{x}) = \sum_i w_i \mathbbm{1} \{ Y_i = c\}$, where:

\begin{align*}
w_i(\mathbf{x}) &= \frac{K(\|\mathbf{x} - \mathbf{x}_i\| / h)}{\sum_j K(\|\mathbf{x} - \mathbf{x}_j\| / h)} \text{ if } \mathcal{B}(\mathbf{x},h) \cap \mathbf{x} \neq \phi, \\
w_i(\mathbf{x}) &= \frac{1}{n} \text{ otherwise }
\end{align*}

\noindent
Note that for $k$-NN, $w_i(\mathbf{x}) = \frac{1}{k}$ and for $h$-NN, $w_i(\mathbf{x}) = \frac{1}{|\mathcal{B}(\mathbf{x},h)|}$. While estimating gradients, we actually work with the softmaxed output 
$$ \bar{f}_{n,h,i} (\mathbf{x}) = \frac{\exp (\bar{f}_{n,h,i} (\mathbf{x})) }{\sum_j \exp (\bar{f}_{n,h,j} (\mathbf{x}))}   $$

\section{Notation and Setup}
For a vector $x\in \mathbb{R}^d$, we denote the euclidean norm as $\|x\|$. For a matrix, we denote the spectral norm, which is the largest singular value of the matrix $\sigma_{\max}(A)$ as $\|A\|_2$. The column space of a matrix $A \in R^{n \times m}$ is denoted as $\text{im}(A)$ where $\text{im}(A) = \{\y \in \R^n | \y = A\x \text{ for some } \x \in \R^m\}$, and $\text{ker}(A)$ is used to denote the null space of matrix $A \in R^{n \times m}$: $\text{ker}(A) = \{\x \in \R^m | A\x = 0 \}$. We use $A \circ B$ to denote the Hadamard product of matrices $A$ and $B$. Let the estimated nonparametric function be $f_{n,h,c}(\x) = \sum_{i} \omega_i(\x) \mathbb{1} \{y_i = c\}$, and $\tilde{f}_{n,h,c}(\x) = \sum_{i} \omega_i(\x) \P(y_i = c|x_i)$. Our estimated gradient at dimension $i$ is given as $$\displaystyle \Delta_{t,i}f_{n,h,c}(\x) = \frac{f_{n,h,c}(\x + t e_i) - f_{n,h,c}(\x - t e_i)}{2t},$$ and the estimated and true gradients for class $c$ are given as: 
$$\hat{\nabla} f_{n,h,c}(\x) = \begin{bmatrix}
\Delta_{t,1}f_{n,h,c}(\x) \cdot \mathbb{1}_{A_{n,1} (\x)} \\
\Delta_{t,2}f_{n,h,c}(\x) \cdot \mathbb{1}_{A_{n,2} (\x)} \\
...\\
\Delta_{t,d}f_{n,h,c}(\x) \cdot \mathbb{1}_{A_{n,d} (\x)}
\end{bmatrix}, \hat{\nabla} f_c(\x) = \begin{bmatrix}
\Delta_{t,1} f_c(\x) \cdot \mathbb{1}_{A_{n,1} (\x)} \\
\Delta_{t,2} f_c(\x) \cdot \mathbb{1}_{A_{n,2} (\x)} \\
...\\
\Delta_{t,d} f_c(\x) \cdot \mathbb{1}_{A_{n,d} (\x)}
\end{bmatrix}$$

Where $A_{n,i}(X)$ is the event that enough samples contribute to the estimate $\Delta_{t,i} f_{n,h}(X)$: $$A_{n,i}(X) \equiv \min_{\{t,-t\}} \mu_n(B(X+s e_i, h/2)) \geq \frac{2d \ln 2n + ln(4/\delta)}{n}$$
$$A_{i}(X) \equiv \min_{\{t,-t\}} \mu(B(X+s e_i, h/2)) \geq 3 \cdot \frac{2d \ln 2n + ln(4/\delta)}{n}$$ $\mu_n$, $\mu$ are empirical mass and mass of a ball, respectively. We denote indicators for events $A_{n,i}(X)$ and $A_{i}(X)$ as:

$\mathbb{I}_n(x) = \begin{bmatrix}
\mathbb{1}_{A_{n,1} (\x)} \\
\mathbb{1}_{A_{n,2} (\x)} \\
...\\
\mathbb{1}_{A_{n,d} (\x)}
\end{bmatrix}$, $\overline{\mathbb{I}_n(x)} = \begin{bmatrix}
\mathbb{1}_{\bar{A}_{n,1} (\x)} \\
\mathbb{1}_{\bar{A}_{n,2} (\x)} \\
...\\
\mathbb{1}_{\bar{A}_{n,d} (\x)}
\end{bmatrix}$, $\mathbb{I}(x) = \begin{bmatrix}
\mathbb{1}_{A_{1} (\x)} \\
\mathbb{1}_{A_{2} (\x)} \\
...\\
\mathbb{1}_{A_{d} (\x)}
\end{bmatrix}$, $\overline{\mathbb{I}(x)} = \begin{bmatrix}
\mathbb{1}_{\bar{A}_{1} (\x)} \\
\mathbb{1}_{\bar{A}_{2} (\x)} \\
...\\
\mathbb{1}_{\bar{A}_{d} (\x)}
\end{bmatrix}$.

The Jacobian outer product matrix is $G(\x) = \mathbf{J}_f(\x) \mathbf{J}_f(\x)^T$, with the estimated Jacobian matrix being:
$$\hat{\mathbf{J}}_{f}(\x) = 
\begin{bmatrix} 
\hat{\nabla} f_{n,h,1}(\x) & \hat{\nabla} f_{n,h,2}(\x) & \dots & \hat{\nabla} f_{n,h,k}(\x)
\end{bmatrix}$$ the EJOP is denoted $\hat{G}(\x) = \hat{\mathbf{J}}_f(\x) \hat{\mathbf{J}}_f(\x)^T$.

\subsection{Assumptions}
We make minimalistic assumptions: Simply that $f$ is continuously differentiable and  $\mu$ has lower-bounded density on a compact support $\Xx$. All of the detailed assumptions below will then hold:

\begin{itemize}
	\item Noise: Let $\eta(X) \triangleq Y-f(X)$. We assume $\forall \delta>0 \text{ there exists } p>0 \text{ such that } \sup_{x\in \X} \prf{Y | X= x}{\abs{\eta(x)}>p}\leq \delta.$ The infimum over all $p$ is denoted by $C_Y(\delta)$. Moreover, the variance of $(Y|X=x)$ is upper-bounded by a constant $\sigma^2_Y$ uniformly over $x\in \X$.
	\item Bounded Gradient: Defining the $\tau$-\emph{envelope} of $\Xx$ as $\Xx+B(0, \tau) \triangleq \braces{z\in B(x, \tau), x\in \Xx}$. We assume there exists
	$\tau$ such that $f$ is continuously differentiable on the $\tau$-envelope $\Xx+B(0, \tau)$. Furthermore, for all
	$x\in \Xx+B(0, \tau), k \in [c]$, we have $\norm{\nabla f_k(x)}\leq R$ for some $R>0$, and $\nabla f$ is uniformly continuous on $\Xx+B(0, \tau)$.
	\item Modulus of continuity of $\nabla f_k$: Let $\epsilon_{t,k,i} = \sup_{\x \in \Xx, s \in [-t,t]} \left| \frac{\partial f_k(\x)}{\partial x_i} - \frac{\partial f_k(\x + s e_i)}{\partial x_i} \right|$ and $\epsilon_{t,i} = \max_{k} \epsilon_{t,c,i}$, define the $(t,i)$-boundary of $\mathcal{X}$ as $\partial_{t,i}(\Xx) = \{\x: \{\x+ t e_i, x- te_i\} \not\subseteq \Xx\}$. When $\mu$ has continues density on $\mathcal{X}$ and $\nabla f_k$ is uniformly continuous on $\Xx + B(0,\tau)$, we have $\mu(\partial_{t,i}(\Xx)) \xrightarrow{t \rightarrow 0} 0$ and $\epsilon_{t,k,i} \xrightarrow{t \rightarrow 0} 0$.
\end{itemize}

\section{Consistency of Estimator \texorpdfstring{$ \e_n \hat{G}(X)$} {} %
	of the Jacobian Outerproduct \texorpdfstring{$\e_X G(X)$}{} %
}

To show that the estimator $ \e_n \hat{G}(X)$ is consistent, we proceed to bound $ \| \e_n \hat{G}(X) - \e_X G(X) \|$ for finite $n$, which is encapsulated in the theorem that follows. There are two main difficulties in the proof, which are addressed by a sequence of lemmas. One has to do with the fact that the gradient estimate at any point depends on all other points, and second, having gradient estimates for $c$ classes.

\begin{aside}{Main Result}
	\begin{theorem}
		Let $t + h \leq \tau$, and let $0 \leq \delta \leq 1$. There exist $C = C(\mu,K(\cdot))$ and $N = N(\mu)$ such that the following holds with probability at least $1 - 2\delta$. Define $A(n) = \sqrt{Cd \cdot \log(kn/\delta)} \cdot 0.25/\log^2(n/\delta)$. Let $n \geq N$, we have:
		\begin{small}
			\begin{align*}
			&\| \e_n \hat{G}(X)] - \e_X G(X) \|_2 \leq \frac{6 R^2}{\sqrt{n}}\left(\sqrt{\ln d} + \sqrt{\ln \frac{1}{\delta}} \right) +  k \left(3R + \sqrt{\sum_{i \in [d]} \epsilon^2_{t,i}} + \sqrt{d}\left(\frac{hR + 1}{t} \right) \right) \cdot\\
			&\left[\frac{\sqrt{d}}{t} \sqrt{\frac{A(n)}{nh^d} + h^2 R^2} +  R \left( \sqrt{\frac{d \ln \frac{d}{\delta}}{2n}} + \right. \right. \left. \left. \sqrt{\sum_{i \in [d]} \mu^2(\partial_{t,i}(\mathcal{X}))} \right) + \sqrt{\sum_{i \in [d]} \epsilon^2_{t,i}} \right]
			\end{align*}
		\end{small}
		\label{thm:main}
	\end{theorem}
\end{aside}
\begin{proof}
	We begin with the following decomposition:
	\begin{align*}
	\| \e_n \hat{G}(X) - \e_X G(X) \|_2 \leq& \| \e_n G(X) - \e_X G(X) \|_2 + \| \e_n \hat{G}(X) - \e_n G(X) \|_2
	\end{align*}
	The first term on the right hand side i.e. $\| \e_n G(X) - \e_X G(X) \|_2$ is bounded using Lemma \ref{lemma:rm}; by using Lemma \ref{lemma:matrix} we bound the second term $\| \e_n \hat{G}(X)  - \e_n G(X) \|_2$, this is done with respect to $\sum_{k \in [c]} \e_n \|\nabla f_k(X) - \hat{\nabla} f_{n,h,k}(X)\|_2$; therefore we need to bound $\sum_{k \in [c]} \e_n \|\nabla f_k(X) - \hat{\nabla} f_{n,h,k}(X)\|_2$, which is done by employing Theorem \ref{thm:f} which concludes the proof.
\end{proof}

\noindent
Note that consistency is implied for $t \xrightarrow{n \rightarrow \infty} 0$, $h \xrightarrow{n \rightarrow \infty} 0$, $h/t^2 \xrightarrow{n \rightarrow \infty} 0$, and $(n / \log n)h^d t^4 \xrightarrow{n \rightarrow \infty} \infty$, this is satisfied for many settings, for example $t \propto h^{1/4}$, $h \propto \frac{1}{\ln n}$.

\subsection{Bounding \texorpdfstring{$\| \e_n G(X) - \e_X G(X) \|_2$}{} %
}
To bound this term, we use the following random matrix concentration result.
\begin{lemma}
	\cite{randommatrix,kakadenotes}. For the random matrix $\X \in \R^{d_1 \times d_2}$ with bounded spectral norm $\|\X\|_2 \leq M$, let $d = \min\{d_1,d_2\}$, and $\X_1,\X_2,...,\X_n$ are i.i.d. samples, with probability at least $1 - \delta$, we have
	\begin{eqnarray*}
		\left\| \frac{1}{n} \sum_{i=1}^n \X_i - \e \X \right\|_2 \leq \frac{6 M}{\sqrt{n}} \left(\sqrt{\ln d} + \sqrt{\ln \frac{1}{\delta}} \right)
	\end{eqnarray*}
\end{lemma}
Using the bounded gradient assumption we can apply the above lemma to i.i.d matrices $G(X), X \in \mathbf{X}$, yielding the following lemma. 
\begin{lemma}
	With probability at least $1 - \delta$ over i.i.d sample $X$
	\begin{eqnarray*}
		\| \e_n G(X) - \e_X G(X) \|_2 \leq \frac{6 R^2}{\sqrt{n}} \left(\sqrt{\ln d} + \sqrt{\ln \frac{1}{\delta}} \right)
	\end{eqnarray*}
	\label{lemma:rm}
\end{lemma} 

\noindent
Next we proceed to bound the second term in the decomposition mentioned in the proof of theorem \ref{thm:main}.

\subsection{Bounding \texorpdfstring{$\| \e_n \hat{G}(X) - \e_n G(X) \|_2$}{} %
}
We derive a first bound by the following lemma:
\begin{lemma}
	There exists a constant $c$, such that with probability at least $1 - \delta$:
	\begin{align*}
	& \| \e_n \hat{G}(X)  - \e_n G(X) \|_2 \leq  \sum_{k \in [c]} \e_n \|\nabla f_k(X) - \hat{\nabla} f_{n,h,k}(X)\|_2 \cdot \max_{x \in \mathbf{X}} \|\nabla f_k(X) + \hat{\nabla} f_{n,h,k}(X)\|_2
	\end{align*}
	\label{lemma:matrix}
\end{lemma}

\begin{proof}
	The term on the l.h.s can be written in terms of the gradients for each class:
	\begin{align*}
	& \| \e_n \hat{G}(X)  - \e_n G(X) \|_2 = \| \e_n [\hat{G}(X) - G(X)] \|_2 
	\\ & =  \left\| \sum_{k \in [c]} \e_n [  \nabla f_k(X) \cdot \nabla f_k(X)^T -  \hat{\nabla} f_{n,h,k}(X) \cdot   \right. \left. \hat{\nabla} f_{n,h,k}(X)^T ] \right\|_2 \\ & \leq 
	\sum_{k \in [c]}  \left\| \e_n [  \nabla f_k(X) \cdot \nabla f_k(X)^T - \hat{\nabla} f_{n,h,k}(X) \cdot  \right.\left. \hat{\nabla} f_{n,h,k}(X)^T ] \right\|_2
	\end{align*}
	Note that $\nabla f_k(\x) \cdot \nabla f_k(\x)^T - \hat{\nabla} f_{n,h,k}(\x) \cdot   \hat{\nabla} f_{n,h,k}(\x)^T$ may be rewritten as:	
	\begin{align*}
	&\nabla f_k(\x) \cdot \nabla f_k(\x)^T - \hat{\nabla} f_{n,h,k}(\x) \cdot   \hat{\nabla} f_{n,h,k}(\x)^T =  \frac{1}{2} \cdot (\nabla f_k(\x) + \hat{\nabla} f_{n,h,k}(\x)) \cdot (\nabla f_k(\x) - \hat{\nabla} f_{n,h,k}(\x))^T \\
	+ &\frac{1}{2} \cdot (\nabla f_k(\x) - \hat{\nabla} f_{n,h,k}(\x)) \cdot (\nabla f_k(\x) + \hat{\nabla} f_{n,h,k}(\x))^T 
	\end{align*}	
	Using this yields:
	\begin{align*}
	& \| \e_n \hat{G}(X)  - \e_n G(X) \|_2 \leq  \frac{1}{2} \sum_{k \in [c]} \|\e_n [  (\nabla f_k(X) +  \hat{\nabla} f_{n,h,k}(X)) \cdot (\nabla f_k(X) - \hat{\nabla} f_{n,h,k}(X))^T ] \|_2 \\
	+& \frac{1}{2} \sum_{k \in [c]} \|\e_n [ (\nabla f_k(X) - \hat{\nabla} f_{n,h,k}(X)) \cdot  (\nabla f_k(X) + \hat{\nabla} f_{n,h,k}(X))^T ]\|_2 \\
	=& \sum_{k \in [c]} \| \e_n [(\nabla f_k(X) - \hat{\nabla} f_{n,h,k}(X)) \cdot (\nabla f_k(X) + \hat{\nabla} f_{n,h,k}(X))^T]\|_2
	\end{align*}
	Employing Jensen's inequality:
	\begin{align*}
	& \e_n [(\nabla f_k(X) - \hat{\nabla} f_{n,h,k}(X)) \cdot (\nabla f_ck(X) +  \hat{\nabla} f_{n,h,k}(X))^T]\|_2 \\ & \leq 	\e_n \| (\nabla f_k(X) - \hat{\nabla} f_{n,h,k}(X)) \cdot (\nabla f_k(X) + \hat{\nabla} f_{n,h,k}(X))^T \|_2
	\end{align*}
	\noindent
	combining the above, gives us the following bound on $\| \e_n \hat{G}(X)  - \e_n G(X) \|_2$
	\begin{align*}
	&\| \e_n \hat{G}(X)  - \e_n G(X) \|_2 \leq \sum_{k \in [c]} \e_n \|  (\nabla f_k(X) - \hat{\nabla} f_{n,h,k}(X)) \cdot (\nabla f_k(X) + \hat{\nabla} f_{n,h,k}(X))^T\|_2 \\
	=& \sum_{k \in [c]} \e_n \|\nabla f_k(X) - \hat{\nabla} f_{n,h,k}(X)\|_2 \cdot \|\nabla f_k(X) + \hat{\nabla} f_{n,h,k}(X)\|_2. \\
	\leq& \sum_{k \in [c]} \e_n \|\nabla f_k(X) - \hat{\nabla} f_{n,h,k}(X)\|_2 \cdot \max_{X \in \mathbf{X}} \|\nabla f_k(X) + \hat{\nabla} f_{n,h,k}(X)\|_2.
	\end{align*}
\end{proof}
The above bound has a dependence on $\| \nabla f_k(X) + \hat{\nabla} f_{n,h,k}(X) \|_2$, which we now proceed to bound below:

\subsection{Bounding \texorpdfstring{$\| \nabla f_k(X) + \hat{\nabla} f_{n,h,k}(X) \|_2$}{} %
}
We first bound the max term, by the following lemma:

\begin{lemma}
	$\forall c \in [k]$, we have
	\begin{align*}
	& \max_{X \in \mathbf{X}} \| \nabla f_k(X) + \hat{\nabla} f_{n,h,k}(X) \|_2 \leq 3R + \sqrt{\sum_{i \in [d]} \epsilon^2_{t,i}} + \sqrt{d}\bigg(\frac{hR + 1}{t}\bigg)
	\end{align*}
\end{lemma}
\begin{proof}
	$\forall x \in \mathbf{X}$, we have
	\begin{align*}
	& \| \nabla f_k(\x) + \hat{\nabla} f_{n,h,k}(\x) \|_2  \leq \| \nabla f_k(\x) \|_2 + \| \hat{\nabla} f_{n,h,k}(\x) \|_2 
    \leq 2 \|\nabla f_k(\x) \|_2 + \|\nabla f_k(\x) - \hat{\nabla} f_{n,h,k}(\x)\|_2 
	\\ & \leq 2R + \|\nabla f_k(\x) - \hat{\nabla} f_k(\x) \|_2 + \|\hat{\nabla} f_k(\x) - \hat{\nabla} f_{n,h,k}(\x)\|_2 \\
	\end{align*}
	Next, we adopt the steps as in the proof for Lemma \ref{lemma7}, and get the following bound: 
	\begin{align*}
	& \|\hat{\nabla} f_k(\x) - \hat{\nabla} f_{n,h,k}(\x)\|_2 \leq \sqrt{\sum_{i \in [d]} (|\Delta_{t,i}f_{n,h,k}(\x) - \Delta_{t,i} f_k(\x)| \cdot  \mathbb{1}_{A_{n,i}(\x)})^2},
	\end{align*}
	this is because
	\begin{align*}
	& |\Delta_{t,i}f_{n,h,k}(\x) - \Delta_{t,i} f_k(\x)| \cdot  \mathbb{1}_{A_{n,i}(\x)} \leq \frac{1}{t} 
	\max_{s \in \{-t,t\}} |\tilde{f}_{n,h,k}(\x + s e_i) - f_k(\x + s e_i)| \cdot  \mathbb{1}_{A_{n,i}(\x)} \\
	&+ \frac{1}{t} \max_{s \in \{-t,t\}} |\tilde{f}_{n,h,k}(\x + s e_i) - f_{n,h,k}(\x + s e_i)| \cdot  \mathbb{1}_{A_{n,i}(\x)},
	\end{align*}
	we also know that
	\[
	\max_{s \in \{-t,t\}} |\tilde{f}_{n,h,k}(X + s e_i) - f_{n,h,k}(X + s e_i)|
	\leq 1.
	\] 
	Thus we obtain the following bound:
	\begin{align*}
	\|\hat{\nabla} f_k(X) - \hat{\nabla} f_{n,h,k}(X)\|_2 \leq \sqrt{d}(\frac{hR + 1}{t})
	\end{align*}
	While, we also have that
	\begin{align*}
	& \|\nabla f_k(X) - \hat{\nabla} f_k(X) \|_2 \leq \|\nabla f_k(X) \circ \mathbf{I}_n(X) - \hat{\nabla} f_k(X)\|_2 + \|\nabla f_k(X) \circ \overline{\mathbf{I}_n(X)}\|_2 
	\leq  R + \sqrt{\sum_{i \in [d]} \epsilon^2_{t,i}}
	\end{align*}
	Combining the above completes the proof
\end{proof} 

Next we need to bound $\e_n \|\nabla f_k(X) - \hat{\nabla} f_{n,h,k}(X)\|_2$, which we do so in the next subsection:

\subsection{Bound on \texorpdfstring{$\e_n \|\nabla f_k(X) - \hat{\nabla} f_{n,h,k}(X)\|_2$}{} %
}

We first decompose $\e_n \|\nabla f_c(X) - \hat{\nabla} f_{n,h,k}(X)\|_2$ as:
\begin{align*}
& \e_n \|\nabla f_k(X) - \hat{\nabla} f_{n,h,k}(X)\|_2 \leq 
\e_n \|\nabla f_k(X) - \hat{\nabla} f_k(X)\|_2 + \e_n \|\hat{\nabla} f_k(X) - \hat{\nabla} f_{n,h,k}(X)\|_2
\end{align*}
the first term in the r.h.s of the above i.e. $\e_n \|\nabla f_k(X) - \hat{\nabla} f_k(X)\|_2$ can in turn be decomposed as:
\begin{align*}
& \e_n \|\nabla f_k(X) - \hat{\nabla} f_k(X)\|_2 \leq  \e_n \|\nabla f_k(X) \circ \mathbb{I}_n(X) - \hat{\nabla} f_k(X)\|_2 + \e_n \|\nabla f_k(X) \circ \overline{\mathbb{I}_n(X)}\|_2
\end{align*}

We need to bound both terms that appear on the r.h.s of the above, which we do so in the next two subsections, starting with the second term. 

\subsubsection{Bounding $\e_n \|\nabla f_k(X) \circ \overline{\mathbb{I}_n(X)}\|_2$}

\begin{lemma}
	With probability at least $1 - \delta$ over the choice of $X$:
	\begin{align*}
	& \e_n \|\nabla f_k(X) \circ \overline{\mathbb{I}_n(X)}\|_2 \leq R \left( \sqrt{\frac{d \ln \frac{d}{\delta}}{2n}} + \sqrt{\sum_{i \in [d]} \mu^2(\partial_{t,i}(\mathcal{X}))} \right)
	\end{align*}
	\label{lemma1}
\end{lemma}
\begin{proof}
	We begin by recalling the bounded gradient assumption: $\|\nabla f(\x)\|_2 \leq R$, using which we get
	\begin{eqnarray*}
		\e_n \|\nabla f(X) \circ \overline{\mathbb{I}_n(X)}\|_2 \leq R \e_n \|\overline{\mathbb{I}_n(X)}\|_2
	\end{eqnarray*}
	By relative VC bounds \cite{doi:10.1137/1116025}, if we set $\alpha_n = \frac{2d \ln 2n + \ln(4/\delta)}{n}$, then with probability at least $1 - \delta$ over the choice of $X$, for all balls $B \in R^d$ we have $\mu(B) \leq \mu_n(B) + \sqrt{\mu_n(B) \alpha_n} + \alpha_n$. Thus, with probability at least $1 - \delta$, $\forall i \in [d]$, $\bar{A}_{n,i}(X) \Rightarrow \bar{A}_{i}(X)$. Moreover, since $\|\overline{\mathbb{I}(X)}\|_2 \leq \sqrt{d}$, then by Hoeffding's inequality,
	\begin{eqnarray*}
		\mathbb{P}(\e_n \|\overline{\mathbb{I}(X)}\|_2 - \e_X \|\overline{\mathbb{I}(X)}\|_2 \geq \epsilon) \leq e^{-\frac{2n \epsilon^2}{d}}
	\end{eqnarray*}
	applying the union bound, we have the following with probability at least $1 - \delta$
	\begin{eqnarray*}
		\e_n \|\overline{\mathbb{I}_n(X)}\|_2 \leq \e_n \|\overline{\mathbb{I}(X)}\|_2
		\leq \e_X \|\overline{\mathbf{I}_n(X)}\|_2 + \sqrt{\frac{d \ln \frac{d}{\delta}}{2n}}
	\end{eqnarray*}
	But note that we have: $\e_X \mathbb{1}_{\bar{A}_i(X)} \leq \e_X[\mathbb{1}_{\bar{A}_i(X)} | X \in \mathcal{X} \backslash \partial_{t,i}(\mathcal{X})] + \mu(\partial_{t,i}(\mathcal{X}))$
	\noindent
	to see why this is true observe that $\e_X[\mathbf{1}_{\bar{A}_i(X)} | X \in \mathcal{X} \backslash \partial_{t,i}(\mathcal{X})] = 0$ because $\mu(B(x+se_i,h/2)) \geq C_{\mu} (h/2)^d \geq 3 \alpha$ when we set $h \geq (\log^2(n/\delta)/n)^{1/d}$. So, we have:
	\begin{eqnarray*}
		\e_X \|\overline{\mathbb{I}_n(X)}\|_2 \leq \sqrt{\sum_{i \in [d]} \mu^2(\partial_{t,i}(\mathcal{X}))}
	\end{eqnarray*}
	Thus with probability at least $1 - \delta$, we obtain the following:
	\begin{align*}
	& \e_n \|\nabla f_k(X) \circ \overline{\mathbf{I}_n(X)}\|_2 \leq R \left( \sqrt{\frac{d \ln \frac{d}{\delta}}{2n}} + \sqrt{\sum_{i \in [d]} \mu^2(\partial_{t,i}(\mathcal{X}))} \right)
	\end{align*}
\end{proof}

Next we need to bound the first term that appeared on the r.h.s. of the decomposition of $\e_n \|\nabla f_k(X) - \hat{\nabla} f_k(X)\|_2$, reproduced below for ease of exposition:

\begin{align*}
& \e_n \|\nabla f_k(X) - \hat{\nabla} f_k(X)\|_2 \leq   \e_n \|\nabla f_k(X) \circ \mathbb{I}_n(X) - \hat{\nabla} f_k(X)\|_2 + \e_n \|\nabla f_k(X) \circ \overline{\mathbb{I}_n(X)}\|_2
\end{align*}

\subsubsection{Bounding  \texorpdfstring{$\e_n \|\nabla f_k(X) \circ \mathbb{I}_n(X) - \hat{\nabla} f_k(X)\|_2$}{} %
} 
This bound is encapsulated in the following lemma
\begin{lemma}
	\begin{eqnarray*}
		\e_n \|\nabla f_k(X) \circ \mathbb{I}_n(X) - \hat{\nabla} f_k(X)\|_2 \leq \sqrt{\sum_{i \in [d]} \epsilon^2_{t,c,i}}
	\end{eqnarray*}
	\label{lemma2}
\end{lemma}
\begin{proof}
	We start with the simple observation regarding the envelope: $$f_k(\x + t e_i) - f_k(\x - t e_i) = \int_{-t}^t \frac{\partial f_k(\x + s e_i)}{\partial x_i} ds$$ using this we have
	\begin{align*}
	& 2t \left( \frac{\partial f'_k(\x)}{\partial x_i} - \epsilon_{t,k,i} \right) \leq f_k(\x + t e_i) - f_k(\x - t e_i) \leq  2t \left(\frac{\partial f'_k(\x)}{\partial x_i} + \epsilon_{t,k,i} \right)
	\end{align*}
	Thus we have $$\left|\frac{1}{2t} (f_c(\x + t e_i) - f_c(\x - t e_i)) - \frac{\partial f'_c(\x)}{\partial x_i} \right| \leq \epsilon_{t,c,i}$$ using which we obtain the following
	\begin{align*}
	& \|\nabla f_k(\x) \circ \mathbb{I}_n(x) - \hat{\nabla} f_k(\x)\|_2 =  \sqrt{\sum_{i=1}^d \left| \frac{\partial f'_k(\x)}{\partial x_i} \cdot  \mathbb{1}_{A_{n,i}(\x)} - \Delta_{t,i} f_k(\x) \cdot \mathbb{1}_{A_{n,i}(\x)} \right|^2}
	\\& \leq \sqrt{\sum_{i=1}^d \left| \frac{1}{2t} (f_k(\x + t e_i) - f_k(\x - t e_i)) - \frac{\partial f'_k(\x)}{\partial x_i} \right|^2}
	 \leq \sqrt{\sum_{i \in [d]} \epsilon^2_{t,k,i}}
	\end{align*}
	Taking empirical expectation on both sides finishes the proof.
\end{proof}
Taking a step back, recall again the decomposition of $\e_n \|\nabla f_c(X) - \hat{\nabla} f_{n,h,k}(X)\|_2$:
\begin{align*}
& \e_n \|\nabla f_k(X) - \hat{\nabla} f_{n,h,k}(X)\|_2 \leq
\e_n \|\nabla f_k(X) - \hat{\nabla} f_k(X)\|_2 + \e_n \|\hat{\nabla} f_k(X) - \hat{\nabla} f_{n,h,k}(X)\|_2
\end{align*}
the first term in the r.h.s of the above i.e. $\e_n \|\nabla f_k(X) - \hat{\nabla} f_k(X)\|_2$ was in turn decomposed as:
\begin{align*}
& \e_n \|\nabla f_k(X) - \hat{\nabla} f_k(X)\|_2 \leq  \e_n \|\nabla f_k(X) \circ \mathbb{I}_n(X) - \hat{\nabla} f_k(X)\|_2 + \e_n \|\nabla f_k(X) \circ \overline{\mathbb{I}_n(X)}\|_2
\end{align*}

The analysis in the previous subsection was bounding these two terms individually. Now we turn our attention towards bounding $\e_n \|\hat{\nabla} f_k(X) - \hat{\nabla} f_{n,h,k}(X)\|_2$

\subsubsection{Bounding \texorpdfstring{$\e_n \|\hat{\nabla} f(X) - \hat{\nabla} f_{n,h}(X)\|_2$}{} %
}
First we introduce a lemma which is a modification of Lemma 6 appearing in \cite{DBLP:conf/nips/KpotufeB12}
\begin{lemma}
	Let $t + h \leq \tau$. We have for all $i \in [d]$, and all $s \in \{-t,t\}$:
	$|\tilde{f}_{n,h,c}(\x + s e_i) - f_c(\x + s e_i)| \cdot  \mathbb{1}_{A_{n,i}(\x)} \leq hR$
\end{lemma}
\begin{proof}
	The proof follows the same logic as in \cite{DBLP:conf/nips/KpotufeB12}, with the last step modified appropriately. To be more specific, let $x = X + se_i$, let $v_i = \frac{X_i - x}{\|X_i - x\|_2}$, then we have
	\begin{align*}
	& |\tilde{f}_{n,h,c}(\x + s e_i) - f_c(\x + s e_i)| \leq \sum_{i \in [d]} w_i(x) |f(X_i) - f(x)|
	\\ & =  \sum_{i \in [d]} w_i(x) |\int_{0}^{\|X_i - x\|_2} v_i^T \nabla f(x + t v_i) dt | 
  \leq  \sum_{i \in [d]} w_i(x) \|X_i - x\|_2 \cdot \max_{x' \in \Xx + B(0,\tau)} \|v_i^T \nabla f(x)\|_2
	\\& \leq  \sum_{i \in [d]} w_i(x) \|X_i - x\|_2 \leq hR
	\end{align*}
\end{proof}

\begin{lemma}
	There exist a constant $C = C(\mu,K(\cdot))$, such that the following holds with probability at least $1 - 2 \delta$ over the choice of $X$. Define $A(n) = 0.25 \cdot \sqrt{C d \cdot \ln(kn/\delta)}$, for all $i \in [d], k \in [c]$, and all $s \in \{-t,t\}$:
	\begin{eqnarray*}
		\e_n |\tilde{f}_{n,h,k}(X + s e_i) - f_{n,h,k}(X + s e_i)|^2 \cdot  \mathbb{1}_{A_{n,i}(X)} \leq \frac{A(n)}{nh^d}
	\end{eqnarray*}
\end{lemma}
\begin{proof}
	The proof follows a similar line of argument as made for the proof of Lemma 7 in \cite{DBLP:conf/nips/KpotufeB12}. First fix any $k \in [c]$,
	Assume $A_{n,i}(X)$ is true, and fix $x = X + s e_i$. Taking conditional expectation on $\mathbf{Y}^n = Y_1,...,Y_n$ given $\mathbf{X}^n = X_1,...,X_n$, we have
	\begin{align*}
	& \e_{\mathbf{Y}^n | \mathbf{X}^n} |f_{n,h,k}(x) - \tilde{f}_{n,h,k}(x)|^2 \leq 0.25 \cdot \sum_{i \in [n]} (w_i(x))^2 \leq 0.25 \cdot \max_{i \in [n]} w_i(x)
	\end{align*}
	Use $\mathbf{Y}^n_x$ to denote corresponding $Y_i$ of samples $X_i \in B(x,h)$. 
	
	Next, we consider the random variable $$\psi(\mathbf{Y}^n_x) = |f_{n,h,k}(x) - \tilde{f}_{n,h,k}(x)|^2$$ Let $\mathcal{Y}_\delta$ denote the event that for all $Y_i \in \mathbf{Y}^n$, $|Y_i - f(X_i)|^2 \leq 0.25$. We know $\mathcal{Y}_\delta$ happens with probability at least $1/2$. Thus
	\begin{align*}
	& \P_{\mathbf{Y}^n | \mathbf{X}^n} (\psi(\mathbf{Y^n}_x) > 2 \e_{\mathbf{Y}^n | \mathbf{X}^n} \psi(\mathbf{Y^n}_x) + \epsilon ) \leq  \P_{\mathbf{Y}^n | \mathbf{X}^n} (\psi(\mathbf{Y^n}_x) >  \e_{\mathbf{Y}^n | \mathbf{X}^n, \mathcal{Y}_\delta} \psi(\mathbf{Y^n}_x) + \epsilon ) 
	\\ & \leq  \P_{\mathbf{Y}^n | \mathbf{X}^n, \mathcal{Y}_\delta} (\psi(\mathbf{Y^n}_x) >  \e_{\mathbf{Y}^n | \mathbf{X}^n, \mathcal{Y}_\delta} \psi(\mathbf{Y^n}_x) + \epsilon ) + \delta /2
	\end{align*}
	By McDiarmid's inequality, we have
	\begin{align*}
	& \P_{\mathbf{Y}^n | \mathbf{X}^n, \mathcal{Y}_\delta} (\psi(\mathbf{Y^n}_x) >  \e_{\mathbf{Y}^n | \mathbf{X}^n, \mathcal{Y}_\delta} \psi(\mathbf{Y^n}_x) + \epsilon ) \leq \exp \left\{ -2 \epsilon^2 \cdot \delta_Y^4 \sum_{i \in [n]} w_i^4(x)
	\right\}
	\end{align*}
	The number of possible sets $\mathbf{Y}_x^n$ (over $x \in \mathcal{X}$) is at most the $n$-shattering number of balls in $\R^d$, using Sauer's lemma we get the number is bounded by $(2n)^{d+2}$. By union bound, with probability at least $1 - \delta$, for all $x \in \mathcal{X}$ satisfying $B(x,h/2) \bigcap \mathbf{X}^n \neq \emptyset$, 
	\begin{align*}
	& \psi(\mathbf{Y}^n_x) \leq 2 \e_{\mathbf{Y}^n | \mathbf{X}^n} \psi(\mathbf{Y^n}_x) +  \sqrt{ 0.25 (d+2) \log(n/\delta)  \sum_{i \in [n]} w_i^4(x)} 
	\leq 2 \sqrt{\e_{\mathbf{Y}^n | \mathbf{X}^n} \psi^2(\mathbf{Y^n}_x)} \\ & + \sqrt{ 0.25 (d+2)  \log(n/\delta)  \delta_Y^4 \max_{i \in [n]} w_i^2(x)} \leq \sqrt{Cd \cdot \log(n/\delta) \cdot 0.25/n^2 \mu_n^2(B(x,h/2))}
	\end{align*}
	Take a union bound over $k \in [c]$, and take empirical expectation, we get $\forall k \in [c]$
	\begin{align*}
	& \e_n |\tilde{f}_{n,h,k}(X + s e_i) - f_{n,h,k}(X + s e_i)|^2 \leq \frac{0.25 \cdot \sqrt{C d \cdot \ln(cn/\delta)}}{n} \sum_{i \in [n]} \frac{1}{n(x_i,h/2)}
	\end{align*}
	where $n(x_i,h/2) = n\mu_n(B(x_i,h/2))$ is the number of points in Ball $B(x_i,h/2)$. 
	
	\vspace{2mm}
	\noindent
	Let $\mathcal{Z}$ denote the minimum $h/4$ cover of $\{x_1,...,x_n\}$, which means for any $x_i$, there is a $z \in \mathcal{Z}$, such that $x_i$ is contained in the ball $B(z,h/4)$. Since $x_i \in B(z,h/4)$, we have $B(z,h/4) \in B(x_i,h/2)$.
	We also assume every $x_i$ is assigned to the closest $z \in \mathcal{Z}$, and write $x_i \rightarrow z$ to denote such $x_i$.
	Then we have:
	\begin{eqnarray*}
		\sum_{i \in [n]} \frac{1}{n(x_i,h/2)} = \sum_{z \in \mathcal{Z}} \sum_{x_i \rightarrow z} \frac{1}{n(x_i,h/2)}
		\leq \sum_{z \in \mathcal{Z}} \sum_{x_i \rightarrow z} \frac{1}{n(z,h/4)} \\
		\leq \sum_{z \in \mathcal{Z}} \frac{n(z,h/4)}{n(z,h/4)}
		&=& |\mathcal{Z}| \leq C_\mu (h/4)^{-d}
	\end{eqnarray*}
	Combining above analysis finishes the proof.
\end{proof}

\begin{lemma}
	There exists a constant $C = C(\mu,K(\cdot))$, such that the following holds with probability at least $1 - 2 \delta$. Define $A(n) = 0.25 \cdot \sqrt{C d \cdot \ln(kn/\delta)} $, $\forall k \in [c]$:
	\begin{eqnarray*}
		\e_n \|\hat{\nabla} f_k(X) - \hat{\nabla} f_{n,h,k}(X)\|_2 \leq \frac{\sqrt{d}}{t} \sqrt{\frac{A(n)}{nh^d} + h^2 R^2}
	\end{eqnarray*}
	\label{lemma7}
\end{lemma}
\begin{proof}
	First we can write the following bound for the l.h.s:
	\begin{align*}
	& \e_n \|\hat{\nabla} f_k(X) - \hat{\nabla} f_{n,h,k}(X)\|_2  \leq \e_n \sqrt{\sum_{i \in [d]} |\Delta_{t,i}f_{n,h,k}(X) - \Delta_{t,i} f_k(X)|^2 \cdot  \mathbb{1}_{A_{n,i}(X)}}
	\\ &\leq \sqrt{\sum_{i \in [d]} \e_n |\Delta_{t,i}f_{n,h,k}(X) - \Delta_{t,i} f_k(X)|^2 \cdot  \mathbb{1}_{A_{n,i}(X)}}
	\\ & \leq \sqrt{\sum_{i \in [d]} \frac{1}{t^2} \max_{s \in \{-t,t\}} \e_n |f_{n,h,k}(X + s e_i) - f_k(X + s e_i)|^2 \cdot  \mathbb{1}_{A_{n,i}(X)}}
	\end{align*}
	First observe that:
	\begin{align*}
	& \e_n |f_{n,h,k}(X + s e_i) - f_k(X + s e_i)|^2 \cdot  \mathbb{1}_{A_{n,i}(X)} 
	 \leq \e_n |\tilde{f}_{n,h,k}(X + s e_i) - f_k(X + s e_i)|^2 \cdot  \mathbb{1}_{A_{n,i}(X)} 
	\\ &+ \e_n |\tilde{f}_{n,h,k}(X + s e_i) - f_{n,h,k}(X + s e_i)|^2 \cdot  \mathbb{1}_{A_{n,i}(X)}
	\end{align*}
	Also notice that: $\e_n |\tilde{f}_{n,h,k}(X + s e_i) - f_k(X + s e_i)|^2 \cdot  \mathbb{1}_{A_{n,i}(X)}$ and $\e_n |\tilde{f}_{n,h,k}(X + s e_i) - f_{n,h,k}(X + s e_i)|^2 \cdot  \mathbb{1}_{A_{n,i}(X)}$ can be respectively bounded by two lemmas from above, thus we get with probability at least $1 - 2\delta$
	\begin{eqnarray*}
		\e_n |f_{n,h,k}(X + s e_i) - f_k(X + s e_i)|^2 \leq h^2 R^2 + \sqrt{\frac{A(n)}{nh^d}}
	\end{eqnarray*}
	Combining above we get with probability at least $1 - 2\delta$, $\forall k \in [c]$
	\begin{align*}
	\|\hat{\nabla} f_k(X) - \hat{\nabla} f_{n,h,k}(X)\|_2 \leq \frac{\sqrt{d}}{t} \sqrt{\frac{A(n)}{nh^d} + h^2 R^2}
	\end{align*}
\end{proof}

\noindent
The following theorem provides a bound on $\e_n \|\nabla f(X) - \hat{\nabla} f_{n,h}(X)\|_2$:
\begin{theorem}
	With probability at least $1 - 2 \delta$ over the choice of $X$, we have $\forall k \in [c]$:
	\begin{align*}
	& \e_n \|\nabla f_k(X) - \hat{\nabla} f_{n,h,k}(X)\|_2 \leq \frac{\sqrt{d}}{t} \sqrt{\frac{A(n)}{nh^d} + h^2 R^2} + R \left( \sqrt{\frac{d \ln \frac{d}{\delta}}{2n}} + \sqrt{\sum_{i \in [d]} \mu^2(\partial_{t,i}(\mathcal{X}))} \right) 
	+ \sqrt{\sum_{i \in [d]} \epsilon^2_{t,i}}
	\end{align*}
	\label{thm:f}
\end{theorem}
\begin{proof}
	We start with the now familiar decomposition:
	\begin{align*}
	& \e_n \|\nabla f_k(X) - \hat{\nabla} f_{n,h,k}(X)\|_2  \leq
	\e_n \|\hat{\nabla} f_k(X) - \hat{\nabla} f_{n,h,k}(X)\|_2 + \e_n \|\nabla f_k(X) \circ \mathbb{I}_n(X) - \hat{\nabla} f_k(X)\|_2 \\& + \e_n \|\nabla f_k(X) \circ \overline{\mathbb{I}_n(X)}\|_2
	\end{align*}
	By Lemma \ref{lemma1} we bound $\e_n \|\nabla f(X) \circ \overline{\mathbb{I}_n(X)}\|_2$; by Lemma \ref{lemma2} we bound $\e_n \|\nabla f(X) \circ \mathbb{I}_n(X) - \hat{\nabla} f(X)\|_2$; by Lemma \ref{lemma7} we bound $\e_n \|\hat{\nabla} f(X) - \hat{\nabla} f_{n,h}(X)\|_2$. Combining these results concludes the proof.
\end{proof}

\section{Bounds on Eigenvalues and Eigenspace variations}\label{sec:eigenspace}

In the above section, we established that  $\e_n \hat{G}(X) $ is a consistent estimator of $\e_X G(X) $. In this section, we also establish consistency of its eigenvalues and eigenspaces,. The analysis here is based upon results from matrix perturbation theory \cite{RPT1,RPT2}.

\subsection{Eigenvalues variation}
We consider the following lemma for eigenvalues variation from matrix perturbation theory:

\begin{lemma}
	\cite{RPT1} Suppose both $G$ and $\hat{G}$ are Hermitian matrices of size $d \times d$, and admit the following eigen-decompositions:
	\begin{eqnarray*}
		G = X \Lambda X^{-1}  \quad \text{and} \quad \hat{G} = \hat{X} \hat{\Lambda} \hat{X}^{-1}
	\end{eqnarray*}
	where $X$ and $\hat{X}$ are nonsingular and
	\begin{eqnarray*}
		\Lambda = \text{diag}(\lambda_1,\lambda_2,...\lambda_d) \quad \text{and} \quad
		\hat{\Lambda} = \text{diag}(\hat{\lambda}_1,\hat{\lambda}_2,...\hat{\lambda}_d)
	\end{eqnarray*}
	and $\lambda_1 \geq \lambda_2 \geq ... \geq \lambda_d$, $\hat{\lambda}_1 \geq \hat{\lambda}_2 \geq ... \geq \hat{\lambda}_d$. Thus for any unitary invariant norm $\| \cdot \|$, we have
	\begin{eqnarray*}
		\|\diag(\lambda_1 - \hat{\lambda}_1,\lambda_2-  \hat{\lambda}_2, ..., \lambda_d - \hat{\lambda}_d) \| \leq \| G - \hat{G} \|
	\end{eqnarray*}
	More specifically, when considering the spectral norm, we have
	\begin{eqnarray*}
		\max_{i \in [d]} |\lambda_i - \hat{\lambda}_i| \leq \| G - \hat{G} \|_2
	\end{eqnarray*}
	and when considering the Frobenius norm, we have
	\begin{eqnarray*}
		\sqrt{\sum_{i \in [d]} |\lambda_i - \hat{\lambda}_i|^2} \leq \| G - \hat{G} \|_F
	\end{eqnarray*}
	\label{lemma:eigvalue}
\end{lemma}

Using the above lemma, we obtain the following theorem that bounds the eigenvalue variation:
\begin{aside}{Eigenvalue Variation Bound}
	\begin{theorem}
		Let $\lambda_1 \geq \lambda_2 \geq ... \geq \lambda_d$ be the eigen-values of $\e_X G(X)$, let $\hat{\lambda}_1 \geq \hat{\lambda}_2 \geq ... \geq \hat{\lambda}_d$ be the eigen-values of $\e_n \hat{G}(X)$. There exist $C = C(\mu,K(\cdot))$ and $N = N(\mu)$ such that the following holds with probability at least $1 - 2\delta$. Define $A(n) = \sqrt{Cd \cdot \log(n/\delta)} \cdot C_Y^2(\delta/2n) \cdot \sigma_Y^2/\log^2(n/\delta)$. Let $n \geq N$, we have:
		\begin{align*}
		&\max_{i \in [d]} |\lambda_i - \hat{\lambda}_i| \leq \frac{6 R^2}{\sqrt{n}}(\sqrt{\ln d} + \sqrt{\ln \frac{1}{\delta}}) + \left(3R + \sqrt{\sum_{i \in [d]} \epsilon^2_{t,i}} + \sqrt{d}(\frac{hR + C_Y(\delta)}{t})\right) \cdot \\&
		\left[\frac{\sqrt{d}}{t} \sqrt{\frac{A(n)}{nh^d} + h^2 R^2} + R \left( \sqrt{\frac{d \ln \frac{d}{\delta}}{2n}} + \left.\right. \right. \right.  \left. \left. \sqrt{\sum_{i \in [d]} \mu^2(\partial_{t,i}(\mathcal{X}))} \right) + \sqrt{\sum_{i \in [d]} \epsilon^2_{t,i}} \right]
		\end{align*}
	\end{theorem}
\end{aside}
\begin{proof}
	By Lemma \ref{lemma:eigvalue}, we bound $\max_{i \in [d]} |\lambda_i - \hat{\lambda}_i|$ with respect to $\| \e_n \hat{G}(X) - \e_X G(X) \|_2$; by Theorem \ref{thm:main} we bound $\| \e_n \hat{G}(X) - \e_X G(X) \|_2$.
\end{proof}
\subsection{Eigenspace variation}
First we introduce the following definition:
\begin{definition}
	\textbf{(Angles between two subspaces)} Let $X, \hat{X} \in \R^{d \times k}$ have full column rank $k$. The angle matrix $\Theta(X, \hat{X})$ between $X$ and $\hat{X}$ is defined as:
	\begin{eqnarray*}
		\arccos ((X^TX)^{-\frac{1}{2}} X^T \hat{X} (\hat{X}^T \hat{X} )^{-1} \hat{X}^T X (X^TX)^{-\frac{1}{2}})^{\frac{1}{2}}
	\end{eqnarray*}
	More specifically, when $k = 1$, it reduces to the angle between two vectors: $	\Theta(\x, \hat{\x}) = \arccos \frac{ | \x^T \hat{\x} | }{\|\x\|_2 \|\hat{\x}\|_2}$
\end{definition}

Armed with this definition, we consider the following lemma on eigenspace variation:

\begin{lemma}
	\cite{RPT2} Suppose both $G$ and $\hat{G}$ are Hermitian matrices of size $d \times d$, and admit the following eigen-decompositions:
	\begin{align*}
	& G = \begin{bmatrix}
	X_1 & X_2
	\end{bmatrix}
	\begin{bmatrix}
	\Lambda_1 & 0 \\
	0 & \Lambda_2
	\end{bmatrix}
	\begin{bmatrix}
	X_1^{-1} \\
	X_2^{-1}
	\end{bmatrix}
	\quad \text{ and } \hat{G} = \begin{bmatrix}
	\hat{X}_1 & \hat{X}_2
	\end{bmatrix}
	\begin{bmatrix}
	\hat{\Lambda}_1 & 0 \\
	0 & \hat{\Lambda}_2
	\end{bmatrix}
	\begin{bmatrix}
	\hat{X}_1^{-1} \\
	\hat{X}_2^{-1}
	\end{bmatrix}
	\end{align*}
	where $X = \begin{bmatrix}
	X_1 & X_2
	\end{bmatrix}$ and $\hat{X} = \begin{bmatrix}
	\hat{X}_1 & \hat{X}_2
	\end{bmatrix}$ are unitary. We have
	\begin{eqnarray*}
		\| \sin  \Theta(X_1, \hat{X}_1) \|_2 \leq \frac{\|(\hat{G} - G) X_1\|_2}{\min_{\lambda \in \lambda(\Lambda_1), \hat{\lambda} \in \lambda(\Lambda_2)} |\lambda - \hat{\lambda} | }
	\end{eqnarray*}
	\label{lemma:eig}
\end{lemma}
Using the above lemma, we get the following theorem for eigenspaces variantion:
\begin{aside}{Eigenspace Variation}
	\begin{theorem}
		Write the eigen-decompositions of $\e_X G(X)$ and $\e_n \hat{G}(X)$ as
		\begin{align*}
		& \e_X G(X) = \begin{bmatrix}
		X_1 & X_2
		\end{bmatrix}
		\begin{bmatrix}
		\Lambda_1 & 0 \\
		0 & \Lambda_2
		\end{bmatrix}
		\begin{bmatrix}
		X_1^{-1} \\
		X_2^{-1}
		\end{bmatrix}, \e_n \hat{G}(X) = \begin{bmatrix}
		\hat{X}_1 & \hat{X}_2
		\end{bmatrix}
		\begin{bmatrix}
		\hat{\Lambda}_1 & 0 \\
		0 & \hat{\Lambda}_2
		\end{bmatrix}
		\begin{bmatrix}
		\hat{X}_1^{-1} \\
		\hat{X}_2^{-1}
		\end{bmatrix}
		\end{align*}
		There exist constants $C = C(\mu,K(\cdot))$ and $N = N(\mu)$ such that the following holds with probability at least $1 - 2\delta$. Define $A(n) = \sqrt{Cd \cdot \log(n/\delta)} \cdot C_Y^2(\delta/2n) \cdot \sigma_Y^2/\log^2(n/\delta)$. Let $n \geq N$:
		\begin{small}
			\begin{align*}
			& \| \sin  \Theta(X_1, \hat{X}_1) \|_2 \leq \\& \Bigg( \frac{\|X_1\|_2}{\min_{\lambda \in \lambda(\Lambda_1), \hat{\lambda} \in \lambda(\Lambda_2)} |\lambda - \hat{\lambda} | }\Bigg) \cdot \Bigg( \frac{6 R^2}{\sqrt{n}}(\sqrt{\ln d} + \sqrt{\ln \frac{1}{\delta}}) + \left(3R + \sqrt{\sum_{i \in [d]} \epsilon^2_{t,i}} + \sqrt{d}(\frac{hR + C_Y(\delta)}{t})\right) \cdot \\ &
			\left[\frac{\sqrt{d}}{t} \sqrt{\frac{A(n)}{nh^d} + h^2 R^2} + R \left( \sqrt{\frac{d \ln \frac{d}{\delta}}{2n}} + \right.\right. \left.\left.\sqrt{\sum_{i \in [d]} \mu^2(\partial_{t,i}(\mathcal{X}))} \right) + \sqrt{\sum_{i \in [d]} \epsilon^2_{t,i}} \right] \Bigg)
			\end{align*}
		\end{small}
		\label{eigenspace}
	\end{theorem}
\end{aside}
\begin{proof}
	By Lemma \ref{lemma:eig}, we bound $\| \sin  \Theta(X_1, \hat{X}_1) \|_2$ with respect to $\|X_1 (\e_n \hat{G}(X) - \e_X G(X)) \|_2$, since $\|X_1 (\e_n \hat{G}(X) - \e_X G(X)) \|_2 \leq \|X_1\|_2 \cdot \|\e_n \hat{G}(X) - \e_X G(X)\|_2$, and by Theorem \ref{thm:main} we bound $\| \e_n \hat{G}(X) - \e_X G(X) \|_2$. Combining these concludes the proof.
\end{proof}

\section{Recovery of projected semiparametric regression model}\label{sec:semiparametric}
In this section, we return to the multi-index motivation of the EGOP and EJOP discussed in the introduction to this chapter. For ease of exposition, we restrict our discussion to the EGOP, but the same argument also works for the EJOP. 

Consider the following projected semiparametric regression model:$f(\x) = g(V^T \x)$
where $V \in \R^{d \times r}, r \ll d$ is a dimension-reduction projection matrix, and $g$ is a nonparametric function. Without loss of generality, we assume $V = [v_1,v_2,....,v_r]$, where $v_i \in R^d, i \in [r]$ is a set of orthonormal vectors, and the gradient outer product (GOP) matrix of $g: \e_{X} [\nabla g(V^T X) \cdot \nabla g(V^T X)^T]$ is nonsingular. The following proposition gives the eigen-decomposition of gradient outer product (GOP) matrix of $f$: $\e_X G(\x)$
\begin{proposition}
	Suppose the eigen-decomposition of $\e_{X} [\nabla g(V^T X) \cdot \nabla g(V^T X)^T]$ is given by:
	\begin{eqnarray*}
		\e_{X} [\nabla g(V^T X) \cdot \nabla g(V^T X)^T] = Z \Lambda Z^{-1}
	\end{eqnarray*}
	then we have the following eigen-decomposition of $\e_X G(X)$:
	\begin{eqnarray*}
		\e_X G(X) = \begin{bmatrix}
			VZ & U
		\end{bmatrix}
		\begin{bmatrix}
			\Lambda & 0 \\
			0       & 0
		\end{bmatrix}
		\begin{bmatrix}
			Z^{-1} V^T \\
			U^T
		\end{bmatrix}
	\end{eqnarray*}
	where $U = [u_1,u_2,...,u_{d-r}]$, $u_i \in [d-r]$ is a set of orthonormal vectors in $\text{ker}(V^T)$.
	\label{prop}
\end{proposition}
\begin{proof}
	Since $f(\x) = g(V^T \x)$, we have $\nabla f(\x) = V \nabla g(V^T \x)$. Thus we get:
	\begin{eqnarray*}
		\e_X G(X) = V \e_{X} [\nabla g(V^T X) \cdot \nabla g(V^T X)^T] V^T = VZ \Lambda Z^{-1}V^T
	\end{eqnarray*}
	When we check the eigen-decomposition given in the proposition, the above equation is satisfied. Moreover, since $$\begin{bmatrix}
	VZ & U
	\end{bmatrix} \begin{bmatrix}
	Z^{-1} V^T \\
	U^T
	\end{bmatrix} = \begin{bmatrix}
	Z^{-1} V^T \\
	U^T
	\end{bmatrix} \begin{bmatrix}
	VZ & U
	\end{bmatrix}= I$$ concludes the proof.
\end{proof}

\vspace{2mm}
\noindent
Since $Z$ in the above proposition is nonsingular, we get that $\text{im}(V) = \text{im}(VZ)$, which means that the column space of projection matrix $V$ is exactly the subspace spanned by the top-$r$ eigenvectors of the GOP matrix $\e_X G(X)$. This point has also been noticed by \cite{sliced,RSSB:RSSB341,DBLP:journals/jmlr/WuGMM10}.

\vspace{2mm}
\noindent
Lastly, we need to show that the projection matrix $V$ can be recovered using the estimated GOP matrix. This is captured in the following:

\begin{aside}{Recovery of Semi-parametric model}
	\begin{theorem}
		Suppose the function $f$ we want to estimate has the form $f(\x) = g(V^T \x)$, and $\tilde{V} \in \R^{d \times r}$ is the matrix composed by the top-$r$ eigenvectors of $\e_n \hat{G}(X)$,
		then with probability at least $1 - 2 \delta$:
		\begin{align*}
		& \| \sin  \Theta(V, \tilde{V}) \|_2 \leq \frac{1}{\lambda_{\min}}\Bigg(
		\frac{6 R^2}{\sqrt{n}}(\sqrt{\ln d} + \sqrt{\ln \frac{1}{\delta}} \Bigg) +  \left(3R + \sqrt{\sum_{i \in [d]} \epsilon^2_{t,i}} + \sqrt{d}(\frac{hR + C_Y(\delta)}{t})\right) \cdot \\ &
		\left[\frac{\sqrt{d}}{t} \sqrt{\frac{A(n)}{nh^d} + h^2 R^2} + R \left( \sqrt{\frac{d \ln \frac{d}{\delta}}{2n}} + \right.\right. \left.\left.\sqrt{\sum_{i \in [d]} \mu^2(\partial_{t,i}(\mathcal{X}))} \right) + \sqrt{\sum_{i \in [d]} \epsilon^2_{t,i}} \right]\Bigg)
		\end{align*}
		where $\lambda_{\min}$ is the smallest eigenvalue of $\e_{X} [\nabla g(V^T X) \cdot \nabla g(V^T X)^T]$.
		\vspace{2mm}
		\noindent
		Suppose $\lambda_1, \lambda_2, ..., \lambda_{d-r}$ are the lowest $d - r$ eigenvalues of $\e_n \hat{G}(X)$,
		and with probability at least $1 - 2 \delta$:
		\begin{align*}
		& max_{i \in [d-r]} |\lambda_i| \leq \Bigg(
		\frac{6 R^2}{\sqrt{n}}(\sqrt{\ln d} + \sqrt{\ln \frac{1}{\delta}} \Bigg) + \left(3R + \sqrt{\sum_{i \in [d]} \epsilon^2_{t,i}} + \sqrt{d}(\frac{hR + C_Y(\delta)}{t})\right) \cdot \\ &
		\left[\frac{\sqrt{d}}{t} \sqrt{\frac{A(n)}{nh^d} + h^2 R^2} + R \left( \sqrt{\frac{d \ln \frac{d}{\delta}}{2n}} + \right.\right. \left.\left.\sqrt{\sum_{i \in [d]} \mu^2(\partial_{t,i}(\mathcal{X}))} \right) + \sqrt{\sum_{i \in [d]} \epsilon^2_{t,i}} \right] \Bigg)
		\end{align*}
	\end{theorem}
\end{aside}
\begin{proof}
	We only sketch the proof. First of all, notice that $V$ is a semi-orthogonal matrix, therefore $\|V\|_2 = 1$. When this observation is combined with above proposition and Theorem \ref{eigenspace}, we get a proof of the first part of the theorem. For proving the second part of the theorem, first observe that by proposition \ref{prop}, the lowest $d-r$ eigenvalues of $\e_X {G}(X)$ are all zeros. This observation when combined with lemma \ref{lemma:eigvalue} finishes the proof.
\end{proof}

\section{Classification Experiments}
In this section, we examine the utility of the EJOP as a technique for metric estimation, when used in the setting of non-parametric classification. We consider non-parametric classifiers that rely on the notion of distance, parameterized by a matrix $\mathbf{M} \succeq 0 $, with the squared distance computed as $(\mathbf{x} - \mathbf{x}')^T \mathbf{M}(\mathbf{x} - \mathbf{x}')$. In the experiments reported in this section, we consider three different choices for $\mathbf{M}$: The first is $\mathbf{M} = \mathbf{I}$, which corresponds to the Euclidean distance. Second we consider the case when $\mathbf{M} = \mathbf{D}$, where $\mathbf{D}$ is a diagonal matrix, the notion of distance in this case corresponds to a scaled Euclidean distance. In particular, in the absence of a gradients weights \cite{DBLP:conf/nips/KpotufeB12}, \cite{GWJMLR} like approach for the multiclass case, we instead obtain weights by using the ReliefF procedure~\cite{relieff1992}, which estimates weights for the multiclass case by a series of one versus all binary classifications. We use this as a baseline. Finally, we consider $\mathbf{M} = \e_n {G}_n(X)$, where $\e_n {G}_n(X)$ is the estimated EJOP matrix. In particular, letting $VDV^\top$ denote the spectral decomposition of $\mathbf{M}$, we use it to transform the input $\mathbf{x}$ as $D^{1/2}V^\top \mathbf{x}$ for the distance computation. Next, for a fixed choice of $\mathbf{M}$, we can define nearest neighbors of a query point $\mathbf{x}$ in various ways. We consider the following two cases: First, $k$ nearest neighbors (denoted henceforth as $k$NN) for fixed $k$ and second, neighbors that have distance $\leq h$ for fixed $h$ from the query. We denote this as $h$NN. This corresponds to nonparametric classification using a boxcar kernel.

\subsection{A First Experiment on MNIST}
We first consider the MNIST dataset to test the quality of the EJOP metric. We set aside 10,000 points as a validation set, which is used to obtain the ReliefF weights, as well as for tuning the parameter $t_i$ for $i = 1, \dots, 784$ in the EJOP estimation, as well as for tuning the parameter $h$. For the $k$NN case, we fix $k=7$. While the $t_i$ can be tuned separately for each class, we observed that it doesn't afford significant advantages over using a single $t$. Note that no preprocessing is applied on the images, and the metric estimation, as well as classification is done using the raw images. The results on the test set are illustrated in table \ref{tb:MNISTEJOP}. While MNIST is is a considerably easy task, the improvement given by the use of the EJOP as the distance metric over the plain Euclidean distance is substantial. 

\begin{table}
	\begin{center}
		\begin{tabular}{l*{1}{c}}
			Method              & Error \%  \\
			\hline \hline
			Euclidean & 4.93  \\
			ReliefF & 4.11 \\
			EJOP ($k$NN) & 2.08 \\
			EJOP ($h$NN) & 2.17
		\end{tabular}
		\caption{Error rates on MNIST using EJOP as the underlying metric}
		\label{tb:MNISTEJOP}
	\end{center}
	\vspace{-5mm}
\end{table}
\subsection{Classification Experiments}

Next, we consider the datasets considered in \cite{trivediNIPS14} and \cite{kedem2012non}. First we report experiments using plain Euclidean distance, $h$-NN and $k$-NN when the EJOP is used as the metric. The train/test splits are reported in the table. We split 20 \% of the training portion to tune for $h$, $k$ and $t_i$, the results reported are over 10 random runs.

\begin{table}[!th]
	\centering
	\begin{tabu}{|l|l|r|r|r|r|r|r|r|}
		\hline
		Dataset & d & N & train/test & Euclidean & h-NN & $k$-NN \\
		\hline 
		\hline
		\small{Isolet} & 172 &  7797 & 4000/2000  &  14.17 $\pm$ 0.7 & 10.14 $\pm$ 0.9  & 8.67 $\pm$ 0.6 \\
		\hline
		\small{USPS} & 256 & 9298 & 4000/2000  & 7.87 $\pm$ 0.2  & 7.14 $\pm$ 0.3 & 6.67 $\pm$ 0.4  \\
		\hline
		\small{Letters} & 16 & 20000 &  4000/2000 & 7.65 $\pm$ 0.3 & 5.12 $\pm$ 0.7 & 4.37 $\pm$ 0.4 \\
		\hline
		\small{DSLR} & 800 & 157 & 100/50 & 84.85 $\pm$ 4.8 & 41.13 $\pm$ 2.1 & 35.01 $\pm$ 1.4 \\
		\hline 
		\small{Amazon} & 800  & 958 & 450/450  & 66.17 $\pm$ 2.8  & 41.07 $\pm$ 2.3  & 39.85 $\pm$ 1.5  \\
		\hline 
		\small{Webcam} & 800  & 295 & 145/145  & 61.43 $\pm$ 1.7  & 24.86 $\pm$ 1.2  & 23.71 $\pm$ 2.1  \\
		\hline 
		\small{Caltech} & 800  & 1123 & 550/500  & 85.41 $\pm$ 3.5  & 54.65 $\pm$ 2.6  & 52.86 $\pm$ 3.1  \\
		\hline
		\hline
	\end{tabu}
	\caption{Classification error rates on the datasets used in \cite{kedem2012non} using Euclidean distance, $h$NN and $k$NN while using the EJOP as the metric}
	\label{tb:EJOPsecond}
	\vspace{-5mm}
\end{table}
Next, we report results obtained on the same folds using three popular metric learning methods. In particular, we consider Large Margin Nearest Neighbors (LMNN)~\cite{weinberger2009distance}, Information Theoretic Metric Learning (ITML)~\cite{itml2007davis} and Metric Learning to Rank (MLR)~\cite{mcfee10_mlr}. Since these methods explicitly optimize for the metric over a space of possible metrics, the comparison is manifestly unfair, since in the case of the EJOP, there is only one metric, which is estimated from the training samples. The setup is the same as discussed above, with the following addition for the metric learning methods: We learn the metric for $k=5$, and test is using whatever $k$ that was returned while tuning for the EJOP. We observe that despite its simplicity, EJOP yields performance comparable to the metric learning methods, in some cases returning error rates comparable to those returned by MLR and ITML. 

\begin{table}[!th]
	\centering
	\begin{tabu}{|l|l|r|r|r|r|r|r|r|}
		\hline
		Dataset & h-NN & $k$-NN & ITML & LMNN & MLR \\
		\hline 
		\hline
		\small{Isolet} & 10.14 $\pm$ 0.9  & 8.67 $\pm$ 0.6  & 8.43 $\pm 0.3$ & 5.3 $\pm$ 0.4 & 6.59 $\pm$ 0.3 \\
		\hline
		\small{USPS} & 7.14 $\pm$ 0.3 & 6.67 $\pm$ 0.4  & 6.57 $\pm$ 0.2 & 6.23 $\pm$ 0.5 &  6.76 $\pm$ 0.3 \\
		\hline
		\small{Letters} & 5.12 $\pm$ 0.7 & 4.37 $\pm$ 0.4 & 5 $\pm$ 0.7 &  4.1 $\pm$ 0.4& 17.81 $\pm$ 5.1\\
		\hline
		\small{DSLR} & 41.13 $\pm$ 2.1 & 35.01 $\pm$ 1.4 & 21.65 $\pm$ 3.1 &  29.65 $\pm$ 3.7 & 41.54 $\pm$ 2.3\\
		\hline 
		\small{Amazon} & 41.07 $\pm$ 2.3  & 39.85 $\pm$ 1.5  & 39.83 $\pm$ 3.5 & 33.08 $\pm$ 4.2& 29.65 $\pm$ 2.6\\
		\hline 
		\small{Webcam} & 24.86 $\pm$ 1.2  & 23.71 $\pm$ 2.1  & 15.31 $\pm$ 4.3 & 19.78 $\pm$ 1.5& 27.54 $\pm$ 3.9\\
		\hline 
		\small{Caltech} & 54.65 $\pm$ 2.6  & 52.86 $\pm$ 3.1   & 52.37 $\pm$ 4.2 & 52.15 $\pm$ 3.2& 51.34 $\pm$ 4.5\\
		\hline
		\hline
	\end{tabu}
	\caption{Classification error rates given by the EJOP, and three popular metric learning methods}
	\label{tb:EJOPthird}
	\vspace{-5mm}
\end{table}

Finally, table \ref{tb:EJOPfourth} presents results comparing classification error rates obtained by using LMNN and MLR in the plain vanilla case and when the estimated EJOP is used as an initialization metric, showing that it improves performance. We also observed that the convergence of the two methods was also much faster when EJOP was used as an initialization, indicating that the EJOP serves as a good prior. 

\begin{table}[!th]
	\centering
	\begin{tabu}{|l|l|r|r|r|r|r|r|}
		\hline
		Dataset &LMNN & LMNN (EJOP Init.) & MLR & MLR (EJOP init.) \\
		\hline 
		\hline
		\small{Isolet} & 5.3 $\pm$ 0.4  & 4.9 $\pm$ 0.3  &  6.59 $\pm$ 0.3 &  6.11 $\pm$ 0.2 \\
		\hline
		\small{USPS} & 6.23 $\pm$ 0.5 & 5.96 $\pm$ 0.3  & 6.76 $\pm$ 0.3 & 6.21 $\pm$ 0.3  \\
		\hline
		\small{Letters} & 4.1 $\pm$ 0.4 & 4 $\pm$ 0.3 & 17.81 $\pm$ 5.1 &  15.17 $\pm$ 4.6 \\
		\hline
		\small{DSLR} & 29.65 $\pm$ 3.7 & 26.13 $\pm$ 3.9 & 41.54 $\pm$ 2.3 &  38.61 $\pm$ 2.9 \\
		\hline 
		\small{Amazon} & 33.08 $\pm$ 4.2  & 31.17 $\pm$ 3.7  & 29.65 $\pm$ 2.6 & 26.19 $\pm$ 2.9 \\
		\hline 
		\small{Webcam} & 19.78 $\pm$ 1.5  & 19.5 $\pm$ 1.7  & 27.54 $\pm$ 3.9 & 24.11 $\pm$ 3.5 \\
		\hline 
		\small{Caltech} & 52.15 $\pm$ 3.2  & 50.35 $\pm$ 3.4   & 51.34 $\pm$ 4.5 & 50.1 $\pm$ 3.7\\
		\hline
		\hline
	\end{tabu}
	\caption{Classification error rate improvements when EJOP is used as an initialization}
	\label{tb:EJOPfourth}
	\vspace{-8mm}
\end{table}

\section{Conclusion}
In this article we adapted the expected gradient outerproduct to the multi-class setting. We studied its theoretical properties, proposed a simple estimator and showed that it affords improvements in non-parametric classification tasks.

\nocite{trivedithesis}
\nocite{trivediasymmetric}
{\small
	\bibliographystyle{plain}
	\bibliography{example_paper}
}

%

%
%

\end{document}